\documentclass[]{Liebert_Author}

\usepackage{indentfirst}

\title{Safe Supervisory Control of Soft Robot Actuators}

\author{Andrew P. Sabelhaus,$^{1\ast}$ Zach J. Patterson,$^{2}$ Anthony T. Wertz$^{3}$,\\ Carmel Majidi$^{2,3}$\\
{$^{1}$Department of Mechanical Engineering, Boston University,}\\
{110 Cummington Mall, Boston, MA 02215, USA}\\
{$^{2}$Department of Mechanical Engineering, Carnegie Mellon University,}\\
{5000 Forbes Ave, Pittsburgh, PA 15213, USA}\\
{$^{3}$Robotics Institute, Carnegie Mellon University,}\\
{5000 Forbes Ave, Pittsburgh, PA 15213, USA}\\
{$^\ast$To whom correspondence should be addressed;}\\
{E-mail:  asabelha@bu.edu.}
}

\begin{document} 

\maketitle 


\begin{abstract}
  Although soft robots show safer interactions with their environment than traditional robots, soft mechanisms and actuators still have significant potential for damage or degradation particularly during unmodeled contact. 
  This article introduces a feedback strategy for safe soft actuator operation during control of a soft robot.
  To do so, a supervisory controller monitors actuator state and dynamically saturates control inputs to avoid conditions that could lead to physical damage. 
  We prove that, under certain conditions, the supervisory controller is stable and verifiably safe.
  We then demonstrate completely onboard operation of the supervisory controller using a soft thermally-actuated robot limb with embedded shape memory alloy (SMA) actuators and sensing.  
  Tests performed with the supervisor verify its theoretical properties and show stabilization of the robot limb's pose in free space.
  Finally, experiments show that our approach prevents overheating during contact (including environmental constraints and human contact) or when infeasible motions are commanded. 
  This supervisory controller, and its ability to be executed with completely onboard sensing, has the potential to make soft robot actuators reliable enough for practical use.
\end{abstract}

\section{Introduction}

One of the most prevalent claims about soft robots is their intrinsic safety when interacting with humans or the environment~\cite{laschi_soft_2016,majidi_soft_2014}.
Less commonly discussed are new challenges in safety introduced by the novel soft actuators~\cite{zhang_robotic_2019} required for generating motion.
For rigid robots, typical electromagnetic actuators (motors) are of little concern in comparison to the robot body's inertia in damaging its surroundings or causing injury~\cite{vasicSafetyIssuesHumanrobot2013,dhillonRobotSystemsReliability2002}.
In contrast, soft actuators can fail dramatically, as practitioners may recognize.
Informally, pneumatic balloons can pop~\cite{terrynSelfhealingSoftPneumatic2017}, thermal actuators can overheat and cause fire risks or burns to human skin~\cite{soother_challenges_2020}, and dielectrics can cause dangerous arcing~\cite{planteLargescaleFailureModes2006,bilodeauSelfHealingDamageResilience2017}, among others.
As of yet, these risks have been mitigated by simple bespoke system designs, hard limits on actuation input~\cite{yeeharntehArchitectureFastAccurate2008}, or open-loop actuation~\cite{patterson_untethered_2020,Huang2019}. 
Incorporating automatic control into soft robots demands more generalizable and robust approaches to actuator safety.

This article proposes a feedback control framework that ensures safety of a class of soft robot actuators.
The framework employs a model-based supervisor in tandem with a primary, unspecified, control strategy - which we term the \textit{pose} controller (Fig.~\ref{fig:overview}(e)-(g)).
We demonstrate our framework on a thermal shape memory alloy (SMA) actuator with two different pose controllers in the presence of environmental contact (Fig.~\ref{fig:overview}(a)-(d)). 
This task presents a generalizable challenge as the cause of failure (excess heat) can only be indirectly monitored and controlled.

\begin{figure}[htbp]
    \vspace{0.5cm}
    \centering
    \includegraphics[width=6in]{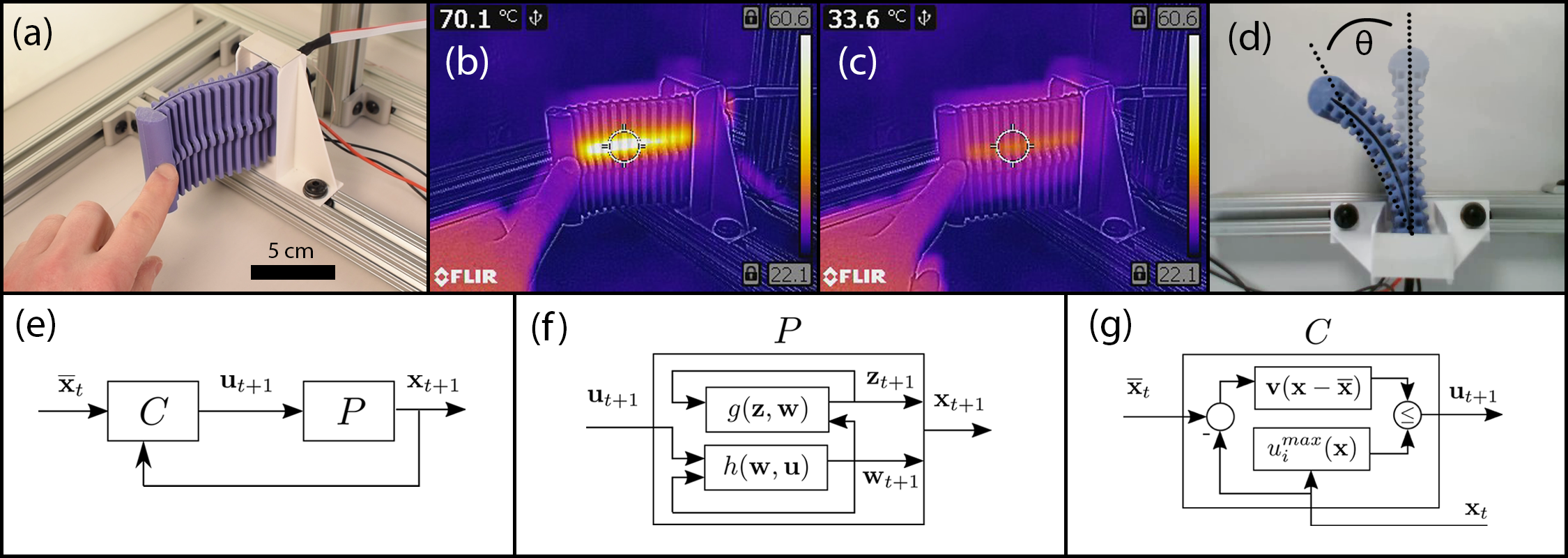}
    \caption{(a) When feedback controllers for soft robots encounter unmodeled physical interactions, such as with the environment or humans, the resulting contact loads or kinematic constraints can cause damage to the robots themselves. (b,c) Our approach prevents these unsafe situations - in our case, our supervisory controller prevents overheating. (d) The example testbed is a soft limb with a single degree of freedom. (e) Our controller operates in a feedback framework that (f) assumes separate actuator dynamics vs. body dynamics. (g) Our supervisor operates in conjunction with any arbitrary underlying controller to apply safe inputs.}
    \label{fig:overview}
    \vspace{0.5cm}
\end{figure}

Specifically, this article contributes:

\begin{enumerate}
    \item A provably-safe, provably-stable supervisory controller for soft robot actuators modeled as affine systems,
    \item A provably-safe integration of the supervisor with any underlying pose controller, and
    \item A verification of the framework on a soft robot limb, maintaining safe actuator states, in an otherwise dangerous task.
\end{enumerate}

\subsection{Background: Robot Safety}

Reliable use of robots in practical settings requires maintaining safe operation and consistent performance throughout their lifespan, regardless of environmental contact or human interaction~\cite{vasicSafetyIssuesHumanrobot2013,dhillonRobotSystemsReliability2002}.
The informal concept of `safety' as limiting force or position~\cite{lasota_survey_2017,zanchettin_safety_2016,tonietti_design_2005} is well-suited for soft robots, since mechanical conformability to contacting surfaces naturally restricts motions.
Even when such limits are exceeded, some soft polymers can self-heal when mechanically damaged and recover their material properties~\cite{terrynSelfhealingSoftPneumatic2017,bilodeauSelfHealingDamageResilience2017}.
However, most soft robots rely on mechanism design for safety~\cite{villanuevaBiomimeticRoboticJellyfish2011,coloradoBiomechanicsSmartWings2012,whiteSoftParallelKinematic2018,huangHighlyDynamicShape2019a}, and there are only few examples of computational intelligence for verifiable behaviors~\cite{balasubramanian_fault_2020}.

In contrast, a formal specification of `safety' is a set inclusion problem, where if a robot's state remains within a certain set for all time then it is considered safe: the set is \textit{invariant} under the system's dynamics~\cite{dang_reachability_1998}. 
Computational techniques such as control barrier functions~\cite{ames_control_2017}, model-predictive control~\cite{massera_filho_safe_2017}, and formal verification~\cite{mitra_safety_2003} use this framework for safety verification.
However, each approach comes with computation and implementation challenges, particularly the requirement of an accurate low-dimensional dynamics model, which is a longstanding challenge in soft robotics.

\subsection{Soft Actuator Safety and Degradation}

The unique material properties of actuators used in soft robotics introduce additional safety challenges: catastrophic pressure failures, high-temperatures and fire, or high-voltage arcing.
As of yet, these dangers have been indirectly addressed by reducing actuation force~\cite{stilli_novel_2017} or by introducing design modifications~\cite{shin_new_2014}.
Feedback control has also only indirectly addressed actuator safety, using approaches such as low impedance~\cite{whitney_hybrid_2016} without verification, open-loop planning~\cite{sabelhaus_inverse_2020,wertz_trajectory_2022} only in known environments, or optimization-based control with only state constraints~\cite{sabelhaus_model-predictive_2021}.

For SMA actuators in particular, prior work in safety has focused on sensing \textit{intrinsic} actuator states (i.e. temperature) via inverting a constitutive model~\cite{ho_modeling_2013,russellImprovingResponseSMA1995} and applying a fixed threshold~\cite{kuribayashi_improvement_1991,liuReinforcementLearningControl2019,yeeharntehArchitectureFastAccurate2008,jinStarfishRobotBased2016}, but these have not shown environmental contact~\cite{soother_challenges_2020} or formal verification.
In addition to physical safety, SMA actuators are known to suffer from degradation due to thermal and mechanical cycling~\cite{sofla_cyclic_2008,villanuevaBiomimeticRoboticJellyfish2011,mohdjaniReviewShapeMemory2014,chinMachineLearningSoft2020}, which when viewed as a temperature constraint~\cite{vanhumbeeckCYCLINGEFFECTSFATIGUE1991} can also be formulated as a safe control problem.

\subsection{Approach}

Our framework considers dynamic saturation as a form of supervisory control, motivated by prior work that uses reachability computations to determine `activation' of a supervisor~\cite{zhang_set_2019}.
This framework is a simplistic version of the formal supervisory controller framework~\cite{qi_state-based_2015,colombo_supervisory_2011,dulce-galindo_autonomous_2019} with only a single switching state.

Though our proposed approach makes a number of assumptions about the robot's actuator dynamics, as are relevant to our use of SMA wires, it may be generalizable among a wider class of soft actuators.
This article assumes one state per actuator, that the actuator states are independent, and that actuator dynamics are an affine system~\cite{wertz_trajectory_2022}.
Other soft actuators may also be modeled by a single parameter, such as piston displacement in pneumatic or hydraulic actuators~\cite{marcheseDynamicsTrajectoryOptimization2015,stolzle_piston-driven_2021}, or cable retraction for cable-driven soft robots~\cite{gravagneLargeDeflectionDynamics2003,renda3DSteadystateModel2012,jarrettRobustControlCableDriven2017a}.
Our supervisor may also be adapted for nonlinear actuators via local linearizations and conservative parameter tuning.
This article includes a study of our controller's tuning parameter to assist in its application.

\section{Supervisory Control for a Soft Robotic Actuator}

Derivation of the supervisor's controller proceeds below by presenting a dynamic saturation condition, then combining with a pose controller, and finally verifying safety.

\subsection{System Model}

The physical states of our particular SMA-powered robot include the robot body's bending curvature and the actuators temperatures, described later in Sec.~\ref{sec:sma_model}.
For a formulation that is generally applicable to soft actuators of all kinds, we abstract these states into $\mathbf{x} \in \mathbb{R}^n$, consisting of the actuator states $\mathbf{w} \in \mathbb{R}^m$ and states relating to the remainder of the system $\mathbf{z}$, as in $\mathbf{x} = [\mathbf{z}^\top \mathbf{w}^\top]^\top$, and some inputs $\mathbf{u} \in \mathbb{R}^p$.
We assume dynamics of the form

\begin{equation}\label{eqn:fullsys}
    \mathbf{x}(k+1) = f(\mathbf{x}(k), \mathbf{u}(k)) = \begin{bmatrix} g(\mathbf{z}, \mathbf{w}) \\ h(\mathbf{w}, \mathbf{u}) \end{bmatrix},
\end{equation}

\noindent where the actuator states influence the pose but not vice-versa.
We again suggest that local linearization of coupled actuator-pose systems may enable equations of this form.

Our supervisory control system does not require knowledge of the pose dynamics $g(\mathbf{z}, \mathbf{w})$, so no soft beam mechanics models are required here.
We do assume that the discrete-time actuator dynamics $h(\mathbf{w}, \mathbf{u})$ are known, that our soft robot has one input per actuator state, and actuators that are uncoupled: there are $p=m$ actuators and

\[
w_i(k+1) = h_i(w_i(k), u_i(k)) \quad \quad \forall i = 1\hdots m.
\]

\noindent In addition, we assume that each actuator dynamics function $h(\cdot, \cdot)$ is a linear (affine) system of the form

\begin{equation}\label{eqn:act_affine}
    h_i(w_i(k), u_i(k)) = a_{1,i} w_i(k) + a_{2,i} u_i(k) + a_3.
\end{equation}

Dropping the $i$ indexing, consider each actuator's system dynamics individually.
We can use the affine augmentation $\widetilde{\mathbf{w}} = [w, \; \; 1]^\top$ to rewrite eqn.~(\ref{eqn:act_affine}) as

\begin{equation}\label{eqn:act_linsys}
    \widetilde{\mathbf{w}}(k+1) = \mathbf{A} \widetilde{\mathbf{w}}(k) + \mathbf{B} u(k)
\end{equation}

\noindent where

\begin{equation}
    \mathbf{A} = \begin{bmatrix} a_1 & a_3 \\ 0 & 1 \end{bmatrix}, \quad \quad \mathbf{B} = \begin{bmatrix} a_2 \\ 0\end{bmatrix}. \label{eqn:bAbB}
\end{equation}

\noindent The dynamics for each actuator, eqn.~(\ref{eqn:act_linsys}), are a linear single-input system, and we therefore can use linear system control techniques for the actuator itself despite the presumably nonlinear body dynamics $g(\mathbf{z}, \mathbf{w})$.

\subsection{The Supervisor's Saturating Controller}

Our problem statement considers a constraint on the actuator state of the form $w \leq w^{MAX}$, i.e., a maximum operating limit.
In the augmented form, this limit is $\widetilde{\mathbf{w}}^{MAX} = [w^{MAX}, \; \; 1]^\top$.
Consider first what magnitude of our input it would take to reach an arbitrary setpoint of this form, $\widetilde{\mathbf{w}}^{SET}$, at the next point in time given an observed state $\widetilde{\mathbf{w}}(k)$:

\begin{equation}\label{eqn:onestep_set}
    \widetilde{\mathbf{w}}(k+1) = \widetilde{\mathbf{w}}^{SET}.
\end{equation}

\noindent From linear systems theory, we have the celebrated result~\cite{baggio_data-driven_2019} that if $\widetilde{\mathbf{w}}^{SET}$ is reachable from $\widetilde{\mathbf{w}}(k)$, it can be steered there in $K$-many steps with minimum-energy cost by applying the input sequence

\begin{equation}
    u^*(t) = \mathbf{B}^\top (\mathbf{A}^\top)^{K-t-1} \mathbf{W}_K^\dag (\widetilde{\mathbf{w}}^{SET} - \mathbf{A}^K \widetilde{\mathbf{w}}(t))
\end{equation}

\noindent at each timestep $t$ in the horizon, where $\mathbf{W}_K$ is the controllability Grammian for discrete-time systems:

\begin{equation}\label{eqn:grammian}
    \mathbf{W}_K = \sum_{t=k}^{K-1} \mathbf{A}^t \mathbf{B} \mathbf{B}^\top (\mathbf{A}^\top)^t,
\end{equation}

\noindent and $(\cdot)^\dag$ is the pseudoinverse.

For the aggressive case of one step ahead where $K=k+1$, as implied by eqn.~(\ref{eqn:onestep_set}), substituting in the definition of the Grammian from eqn.~(\ref{eqn:grammian}) at the single timestep of $t=k$ becomes

\begin{equation}
    u^*(k) = \mathbf{B}^\top (\mathbf{B}\mathbf{B}^\top)^\dag (\widetilde{\mathbf{w}}^{SET} - \mathbf{A} \widetilde{\mathbf{w}}(k)).
\end{equation}

Observe next that the actuator dynamics in eqn.~(\ref{eqn:act_affine}) are a monotone control system~\cite{angeli_monotone_2003}, i.e., for two different inputs $u_1$ and $u_2$ applied at the same known state $w$, dropping time index for brevity,

\begin{align}
    u_1 \leq u_2 \quad & \Rightarrow \quad h(w, u_1) \leq h(w, u_2), \notag \\
    & \Rightarrow \quad w_1(k+1) \leq w_2(k+1).
\end{align}

\noindent Intuitively, our actuator's state has a lower value if we apply less input: less electrical power applied to our thermal muscles means lower temperature. 
Sec. 1 of the Supplementary Information S1 formally shows monotonicity.

A concept, then, to bring our system as close to $\widetilde{\mathbf{w}}^{SET}$ as possible without exceeding $\widetilde{\mathbf{w}}^{SET}$ would be to apply some fraction of this one step max, since $u(k) < u^{MAX}(k) \; \Rightarrow \; \widetilde{\mathbf{w}}(k+1) < \widetilde{\mathbf{w}}^{MAX}$.
Choose a scalar multiplier $\gamma \in (0,1)$ for $u^{MAX}$, which then gives a candidate feedback controller as

\begin{equation}
    u(k) = \gamma \mathbf{B}^\top (\mathbf{B}\mathbf{B}^\top)^\dag (\widetilde{\mathbf{w}}^{SET} - \mathbf{A} \widetilde{\mathbf{w}}(k)).
\end{equation}

\noindent Closing the loop with this controller produces the autonomous dynamics of

\begin{align}
    \widetilde{\mathbf{w}}(k+1) & = (1-\gamma) \mathbf{A} \widetilde{\mathbf{w}}(k) + \gamma \widetilde{\mathbf{w}}^{SET}. \label{eqn:wset_cl}
\end{align}

Make the following substitution to analyze this system as a linear system, eliminating the affine setpoint offset:

\begin{equation}
    \widetilde{\mathbf{w}}' \vcentcolon= \widetilde{\mathbf{w}} - \left(\mathbf{I} - (1-\gamma)\mathbf{A} \right)^{-1} \gamma \widetilde{\mathbf{w}}^{SET}.
\end{equation}

\noindent Substitution into eqn.~(\ref{eqn:wset_cl}) gives

\begin{equation}\label{eqn:cl_wtildeprime}
    \widetilde{\mathbf{w}}'(k+1) = (1-\gamma) \mathbf{A} \widetilde{\mathbf{w}}'(k).
\end{equation}

\noindent The equilibrium point under consideration for the closed-loop system of eqn.~(\ref{eqn:cl_wtildeprime}) is therefore

\[
\widetilde{\mathbf{w}}'=0 \quad \Rightarrow \quad \widetilde{\mathbf{w}} = \left(\mathbf{I} - (1-\gamma)\mathbf{A} \right)^{-1} \gamma \widetilde{\mathbf{w}}^{SET}.
\]

Notice this equilibrium point is \textit{not} equal to our setpoint ($\widetilde{\mathbf{w}}^{SET}$). 
The inclusion of $\gamma$ scales the equilibrium point in addition to slowing convergence.
However, we can then select our $\widetilde{\mathbf{w}}^{SET}$ such that our equilibrium point becomes the constraint boundary $\widetilde{\mathbf{w}}^{MAX}$,

\begin{align}
    \widetilde{\mathbf{w}}^{MAX} & = \left(\mathbf{I} - (1-\gamma)\mathbf{A} \right)^{-1} \gamma \widetilde{\mathbf{w}}^{SET}, \\
    \therefore \widetilde{\mathbf{w}}^{SET} & = \frac{1}{\gamma} \left(\mathbf{I} - (1-\gamma)\mathbf{A} \right) \widetilde{\mathbf{w}}^{MAX}. \label{eqn:wset_substitute}
\end{align}

Combining, the full form of our supervisor's feedback controller that takes the actuator state to the boundary of its constraint with a convergence rate of $\gamma$ is




\begin{align}
    u(k)^{MAX} =  \gamma \mathbf{B}^\top (\mathbf{B}\mathbf{B}^\top)^\dag \left(\frac{1}{\gamma} (\mathbf{I} - (1-\gamma)\mathbf{A} ) \widetilde{\mathbf{w}}^{MAX} - \mathbf{A} \widetilde{\mathbf{w}}(k) \right). \label{eqn:supervisor}
\end{align}

Lastly, we find the error dynamics by defining $\mathbf{e} \vcentcolon= \widetilde{\mathbf{w}} - \widetilde{\mathbf{w}}^{MAX}$, and with some algebra on eqns. (\ref{eqn:cl_wtildeprime})-(\ref{eqn:wset_substitute}), the closed-loop response is

\begin{equation}\label{eqn:cl_err}
    \mathbf{e}(k+1) = (1-\gamma) \mathbf{A} \mathbf{e}(k).
\end{equation}

This is the same as eqn.~(\ref{eqn:cl_wtildeprime}), so the closed-loop system is stable if the open-loop system is stable and $\gamma \in (0,1)$.
See Supplementary Information S1 Sec. 2.1 for a proof.

\subsection{Safety Verification of the Supervisor's Controller}

Operating the supervisor's controller gives the dynamics in eqn.~(\ref{eqn:cl_err}).
We assume the closed-loop system has been designed via the criteria in the Supplementary Information S1 to be G.A.S., $\lim_{k\rightarrow \infty}\widetilde{\mathbf{w}}(k) = \widetilde{\mathbf{w}}^{MAX}$.
However, even exponential stability does not necessarily guarantee that

\begin{equation}\label{eqn:safety}
    \widetilde{\mathbf{w}}(k) < \widetilde{\mathbf{w}}^{MAX} \; \forall k \; \in \mathbb{N}^+ .
\end{equation}

\noindent For example, this safety condition would not hold for an underdamped SISO linear system.

Verifying the condition (\ref{eqn:safety}) can be done instead by calculating an invariant set for some given constraint~\cite{gilbert_linear_1991,blanchini_set_1999}.
First, we pose the inequalities of the safety constraint as a polyhedron, noting that in terms of the error dynamics, safe operation holds if $\widetilde{\mathbf{w}} < \widetilde{\mathbf{w}}^{MAX} \Rightarrow \mathbf{e} \leq 0$.
Choose an arbitrary lower bound on the actuator state, $w^{LB} \leq w$, in order to consider operations on a closed set.
This polyhedron can be expressed as

\begin{equation}\label{eqn:safe_set}
    \mathcal{S} = \{\mathbf{e} \; \; | \; \; \mathbf{H} \mathbf{e} \leq \mathbf{h} \},
\end{equation}

\noindent where

\begin{equation}\label{eqn:Hh}
    \mathbf{H} = \begin{bmatrix} -1 & 0 \\ 0 & -1 \\ 1 & 0 \\ 0 & 1 \end{bmatrix}, \quad \quad
    \mathbf{h} = \begin{bmatrix} w^{LB} \\ 0 \\ 0 \\ 0 \end{bmatrix}
\end{equation}

\noindent since the second element of $\widetilde{\mathbf{w}}$ is constrained to be identical to 1.

We then calculate a maximum positive invariant set $\mathcal{O}_{\infty}$, which contains all~\cite{blanchini_set_1999} the invariant sets $\mathcal{O} \subseteq \mathcal{S}$ such that $\widetilde{\mathbf{w}}(0) \in \mathcal{O} \Rightarrow \widetilde{\mathbf{w}}(k) \in \mathcal{O} \; \forall \; k \in \mathbb{N}^+$ given the closed-loop dynamics of eqn.~(\ref{eqn:cl_err}).
This iterative procedure uses the operation Pre($\mathbf{A},\mathcal{O}$), which generates the set of all states which evolve into a set $\mathcal{O}:\{\mathbf{e} \; | \; \mathbf{P} \mathbf{e} \leq \mathbf{p}\}$ under the dynamics $\mathbf{A}$ in one step:

\begin{equation}\label{eqn:pre}
    \text{Pre}(\mathbf{A}, \mathcal{O}) = \{ \mathbf{e} \; | \; \mathbf{P} \mathbf{A} \mathbf{e} \leq \mathbf{p}\}.
\end{equation}

\noindent A set $\mathcal{O}$ is positive invariant under $\mathbf{A}$ if and only if $\mathcal{O} \subseteq \text{Pre}(\mathcal{O})$~\cite{gilbert_linear_1991}.
A more useful condition can be posed in terms of set equivalence:

\[
\mathcal{O} \subseteq \text{Pre}(\mathcal{O}) \; \iff \; \text{Pre}(\mathcal{O}) \cap \mathcal{O} = \mathcal{O}.
\]

\noindent which, with a polyhedral set, can be readily checked by comparing the constraints (e.g. the $\mathbf{P}$ and $\mathbf{p}$ of a set $\mathcal{O}$).

Given the safety constraint $\mathcal{S}$ in eqn.~(\ref{eqn:safe_set})-(\ref{eqn:Hh}), the well-known Alg.~\ref{alg:oinf} returns the set of initial conditions expressed as a polyhedron for which our system will remain within $\mathcal{S}$.
We implemented Alg.~\ref{alg:oinf} using the MPT3 Toolbox~\cite{MPT3} in MATLAB.
Executing Alg.~\ref{alg:oinf} for our particular $\mathbf{A}$ matrix (system identification of the SMA discussed below) with a variety of $\gamma < 1$ produced $\mathcal{O}_{\infty}=\mathcal{S}$ in every case; i.e., the iteration returned in one step.
This expected behavior verifies safety, confirming our intuition that as long as our supervisor's controller activates while the actuator state is $\widetilde{\mathbf{w}} < \widetilde{\mathbf{w}}^{MAX}$, it will remain so (theoretically) for all time.

\begin{algorithm}[h!]
\caption{Maximum Positive Invariant Set Calculation~\cite{gilbert_linear_1991}} \label{alg:oinf}
\SetKwInput{KwInput}{Input}                
\KwInput{$\mathbf{A}$, $\mathcal{S}$}
$\mathcal{O}_0 \gets \mathcal{S}$, \quad \quad $\mathcal{O}_1 \gets \text{Pre}(\mathbf{A}, \mathcal{O}_0) \cap \mathcal{O}_0$\;
\While{$\mathcal{O}_{j} \neq \mathcal{O}_{j-1}$}
{
$\mathcal{O}_{j+1} \gets \text{Pre}(\mathbf{A}, \mathcal{O}_j) \cap \mathcal{O}_j$\;
$j \gets j+1$
}
\Return{$\mathcal{O}_{\infty} = \mathcal{O}_j$}
\end{algorithm}

\subsection{Supervisor Integration with Pose Controller}

The controller in eqn.~(\ref{eqn:supervisor}) drives our actuator states ($\mathbf{w}$) to a maximum safe value.
This section incorporates the supervisor in composition with another controller, which presumably feeds back the full robot state $\mathbf{x}$ including pose.

So, assume there is a feedback controller $\mathbf{v}(\mathbf{x}(k)) = [v_1(\mathbf{x}(k)), \; \hdots, \; v_p(\mathbf{x}(k))]^\top$, developed independently from the supervisor, that would close the loop of eqn.~(\ref{eqn:fullsys}).
Propose instead the following composition of $\mathbf{v}(\mathbf{x})$ and the supervisor:




\begin{align}
    & \mathbf{u}(\mathbf{x}(k)) = \mathbf{u}^s(\mathbf{x}(k)) = [u^s_1(\mathbf{x}(k)), \; \hdots, \; u^s_p(\mathbf{x}(k))]^\top \label{eqn:boldu_s} \\
    & u^s_i(\mathbf{x}(k)) =
    \begin{cases}
        v_i(\mathbf{x}(k)) & \text{if} \quad v_i(\mathbf{x}(k)) \leq u_i^{MAX}(\mathbf{x}(k))\\
        u_i^{MAX}(\mathbf{x}(k)) & \text{else}
    \end{cases} \label{eqn:u_s}
\end{align}

\noindent where $u_i^{MAX}(\mathbf{x}(k))$ replaces the time indexing in eqn.~(\ref{eqn:supervisor}) with feedback, noting the actuator state $\widetilde{\mathbf{w}}$ is part of the full state $\mathbf{x}$.

We note that this composed system is both continuous (in $\mathcal{C}^0$) and Lipschitz continuous under certain conditions on $v_i(\mathbf{x})$.
See Supplementary Information S1 Sec. 2.2 for a proof.

Finally, we show the most important property of $u^s$, the safety verification that motivates all the work in this article.

\begin{theorem}
{\normalfont Safety.} Consider the closed-loop system $\mathbf{x}(k+1) = f(\mathbf{x}(k), \mathbf{u}^s(\mathbf{x}(k))) = f^{CL}(\mathbf{x})$ defined by eqns. (\ref{eqn:fullsys}), (\ref{eqn:boldu_s})-(\ref{eqn:u_s}), and the $u_i^{MAX}$ in eqn.~(\ref{eqn:supervisor}) for all actuators $i=1\hdots m$.
If Alg.~\ref{alg:oinf} verifies that the set $\mathcal{S}$ defined in eqns. (\ref{eqn:safe_set})-(\ref{eqn:Hh}) is positively invariant for the supervisor's closed-loop error dynamics of eqn.~(\ref{eqn:cl_err}), then $\mathcal{S}$ is also invariant under $f^{CL}(\mathbf{x})$, and




\begin{align*}
    (\widetilde{\mathbf{w}}_i(0) - & \widetilde{\mathbf{w}}_i^{MAX})  \in \mathcal{S} \Rightarrow \widetilde{\mathbf{w}}_i(k) < \widetilde{\mathbf{w}}_i^{MAX} \; \; \forall k \in \mathbb{N}^+.
\end{align*}

\end{theorem}

\begin{proof}
Invariance of $\mathcal{S}$ under the action of the supervisor's controller alone, given by assumption, has proven that applying $u_i^{MAX}(\mathbf{x}(k))$ gives

\begin{equation}
    u_i^{MAX}(\mathbf{x}(k)) \Rightarrow \widetilde{\mathbf{w}}_i(k+1) \leq \widetilde{\mathbf{w}}_i^{MAX}.
\end{equation}

Then by the definition of controller in eqn.~(\ref{eqn:u_s}), $u^s_i(\mathbf{x}) \leq u_i^{MAX}(\mathbf{x})$.
Let the resulting actuator state from applying $u^s_i(\mathbf{x}(k))$ be $\widetilde{\mathbf{w}}_i^s(k+1)$.
From the discussion above, $h(\cdot, \cdot)$ is a monotone control system, and so applying $u^s_i \leq u_i^{MAX}$ gives

\begin{align*}
    u^s_i(\mathbf{x}(k)) \leq u_i^{MAX}(\mathbf{x}) & \Rightarrow \widetilde{\mathbf{w}}_i^s(k+1) \leq \widetilde{\mathbf{w}}_i(k+1) \\
    & \Rightarrow \widetilde{\mathbf{w}}_i^s(k+1) \leq \widetilde{\mathbf{w}}_i^{MAX}.
\end{align*}

\end{proof}

\textit{Remark}.
The above is a simple consequence of the same intuition as with the supervisor's tuning parameter $\gamma$.
For our example application of soft thermal actuators, applying less electrical power causes a lower temperature, and so applying less than a safe maximum of power will guarantee a safe temperature.

\section{Hardware Testbed}

The control system derived above is applicable for any soft robotic system whose equations of motion can be put in the form of eqn.~(\ref{eqn:fullsys})-(\ref{eqn:act_affine}).
As one particular application, this article considers feedback control of a soft robotic limb constructed with thermally-actuated shape memory alloy (SMA) wire coils.
Versions of this limb, previously developed as part of a soft underwater robot~\cite{patterson_untethered_2020}, have recently been deployed by the authors for both open-loop~\cite{wertz_trajectory_2022,sabelhaus_in-situ_2022} and closed-loop~\cite{patterson_robust_2022} control as a free-standing manipulator.
For eventual application in locomotion, this article uses the proposed supervisory controller to maintain safe actuator states when significant and sustained contact occurs.

\subsection{Hardware Design}

Our soft limb consists of a bulk silicone body embedded with actuators and sensors (Fig.~\ref{fig:construction}).
The limb is designed for planar motions only in order to develop algorithms with a reduced-dimensional state space.
The limb's body (Smooth-On Smooth-Sil 945), shown in Fig.~\ref{fig:construction}(a)(1), contains two embedded SMA wire coils (Dynalloy Flexinol, 0.020'' wire diameter). 
The SMA wire coils are constrained by a ridge along the body's horizontal axis, as shown in Fig.~\ref{fig:construction}(a)(2) and Fig.~\ref{fig:construction}(b)(2), so that actuation forces cause bending deflections.

\begin{figure}[htbp]
    \vspace{0.5cm}
    \centering
    \includegraphics[width=6in]{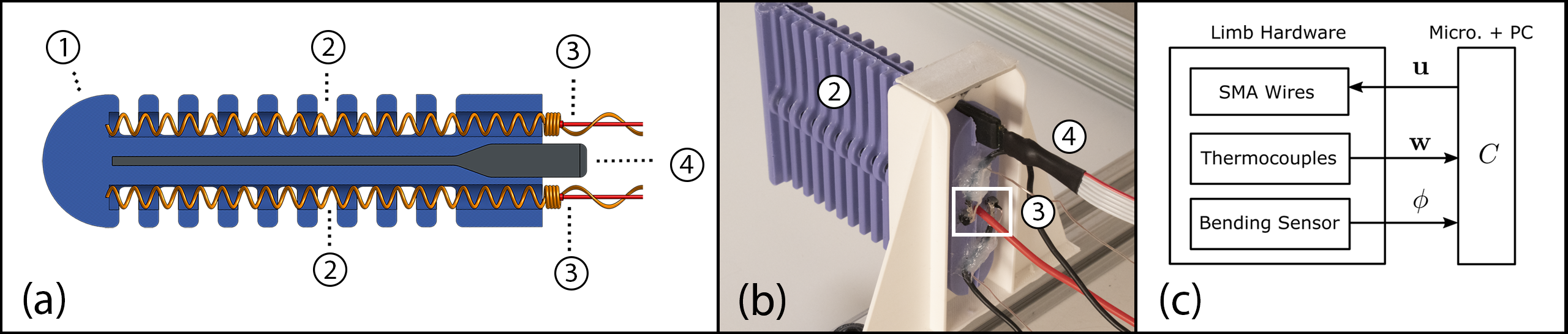}
    \caption{Our hardware testbed consists of a soft robot limb actuated with shape-memory-alloy (SMA) wire coils. A cross section of the limb from Fig.~\ref{fig:overview} shows (a)(1) the limb's bulk body, (a)(2) two antagonistic SMA wire coils, (a)(3) thermocouples attached to SMAs at the rear of the limb, and (a)(4) a soft capacitive bending sensor. Rear angled view of the limb shows (b)(2) the SMAs constrained along a ridge on each side, (b)(3) the thermocouples attached to the SMAs, and (b)(4) the bending sensor in its groove at the top of the limb. A functional block diagram of the limb (c) shows the sensor measurements and control input, as measured and commanded by the controller running on a microcontroller and PC.}
    \label{fig:construction}
    \vspace{0.5cm}
\end{figure}

The two SMA wires are actuated through resistive (Joule) heating.
Current through the wires was controlled using pulse-width modulation (PWM) to N-channel power MOSFETs connected to a 7V power supply. 
A microcontroller sets the PWM duty cycle.

Three sensors are located on the limb: one for the body's pose, and one each for the temperatures of the wires.
Temperature is sensed by thermocouples (Omega Engineering, type K, 30 AWG) affixed to the SMA coils at the rear of the limb using thermally conductive epoxy (MG 8329TCF) via the fabrication procedure described in our prior work~\cite{sabelhaus_in-situ_2022}, see Fig.~\ref{fig:construction}(a)(3) and Fig.~\ref{fig:construction}(b)(3).
A soft capacitive bending sensor (Bendlabs, Inc.) is inserted into a groove in the limb (Fig.~\ref{fig:construction}(a)(4) and Fig.~\ref{fig:construction}(b)(4)), and provides a single measurement of angular deflection of the limb, $\theta(t)$, as shown in Fig.~\ref{fig:overview}(d).

\subsection{Robot and Actuator Model and Calibration}\label{sec:sma_model}

We use a simplified model of the limb for this article, with a state space of

\begin{equation}
    \mathbf{x} = \begin{bmatrix} \theta & \dot \theta & \mathbf{w} \end{bmatrix}^\top \in \mathbb{R}^4, \quad \mathbf{w} = \begin{bmatrix} T_0 & T_1 \end{bmatrix},
\end{equation}

\noindent where the body pose is deflection angle and actuator states are the temperatures of the two wires.
Importantly, this article is not concerned with developing provably-stabilizing controllers for the body pose, and our supervisory controller in eqn.~(\ref{eqn:supervisor}) does not require a model of body pose dynamics $g(\cdot)$.
Therefore, feedback of only the net deflection angle $\theta$ may be sufficient for pose control as in prior work on SMA-powered robots~\cite{Elahinia2002,jin_continuous_2015,ho_modeling_2013,Grant1997}.

The thermal dynamics of our SMA actuators can be approximated in the form of eqn.~(\ref{eqn:act_affine}).
As in prior work~\cite{wertz_trajectory_2022,ho_modeling_2013}, the first-principles model for Joule heating in discrete time is

\begin{equation}
    T_i(k+1) = -\frac{h_c A_c}{C_v}(T_i(k) - T_0)\Delta_t + \frac{1}{C_v} \Delta_t P_i(k)
\end{equation}

\noindent for the $i$-th SMA at time $k$ with specific heat capacity \(C_v\), ambient heat convection coefficient \(h_c\), surface area \(A_c\), and ambient temperature \(T_0\).
The input electrical power, $P_i(k)$, is current controlled with $P=\rho J^2$, where $\rho$ is resistance and $J$ is current density.
For our PWM input, we assume that the duty cycle \(u_i\) modulates the fraction of time current is conducting through the SMA and that current is constant when flowing, so $P_i(k)=\rho J^2 u_i(k)$.
Substituting and factoring out unknown constants into lumped coefficients,

\begin{equation}\label{eqn:joule_affine}
    T_i(k+1) = a_{(1,i)} T_i(k) + a_{(2,i)} u_i(k) + a_{(3,i)}, 
\end{equation}

\noindent just as in eqn.~(\ref{eqn:act_affine}).
Our system input is $\mathbf{u} \in [0,1]^2$.
We calibrate eqn.~(\ref{eqn:joule_affine}) for each SMA from data collected in hardware, using the same procedure as in our prior work on this hardware platform~\cite{wertz_trajectory_2022,sabelhaus_in-situ_2022}.

\section{Pose Feedback for an SMA-Powered Soft Robot}

We employ two representative pose controllers for SMA-actuated robots in order to demonstrate that our supervisory control scheme is agnostic to choice of $\mathbf{v}(\mathbf{x})$.

\subsection{Antagonistic Actuation as a Single-Input, Single-Output System}

Our supervisor is derived for an arbitrary number of soft actuators ($m$) in the form of eqn.~(\ref{eqn:act_affine}).
However, for our particular SMA-powered robot, prior research has shown that a pair of $m=2$ antagonistic actuators can be reformulated as a single-input, single-output (SISO) system~\cite{ho_modeling_2013,prechtl_self-sensing_2020,patterson_robust_2022}.
We first note that the duty cycle input for each wire, $u_i \in [0, 1]$, does not allow for a negative control input: we have no ability to cool the wire.
However, since the actuators are oriented antagonistically, we can map one of the two SMA duty cycles to a negative range of a single scalar input, which we denote $\mu \in [-1, 1]$, 

\begin{align}\label{eqn:siso}
    \mathbf{v}(\mathbf{x}) & = \mathbf{v}(\mu(\mathbf{x})), \\
    \mathbf{v}(\mu(\mathbf{x})) & = \begin{cases}
    \left[ \mu(\mathbf{x}) \; \; \;  0 \right]^\top & \text{if} \quad \mu(\mathbf{x}) \geq 0 \\
    [0 \; \; \; -\mu(\mathbf{x})]^\top & \text{if} \quad \mu(\mathbf{x}) < 0
    \end{cases}
\end{align}

\noindent We therefore only need to specify a (bounded) SISO pose controller, $\mu(\mathbf{x}) : \mathbb{R}^n \mapsto [-1, 1]$.

\subsection{PI with Anti-Windup}

We first test a proportional-integral (PI) controller.
We augment it with an anti-windup (AW) block~\cite{astrom_advanced_2006} since our control $\mu$ saturates at $\pm 1$.
Defining $e(k) = \theta(k) - \bar \theta(k)$ as the difference between current vs. reference bending angle, implicitly indexing into $\mathbf{x}(k)$, and with some minor abuses of notation for clarity, our PI-AW feedback controller takes the form:




\begin{align}
    \eta(e) = & K_p e(k) + K_I \left[ \sum_{\tau=0}^{k-1}\left(e(\tau)\Delta_t + K_{AW}\left( \mu(\tau-1) - \eta(\tau-1) \right) \Delta_t \right)  \right],\\
    \mu(e) = & \; \text{sat}(\eta(e))
\end{align}

\noindent where the linear saturation function $\text{sat}(\cdot) : \mathbb{R} \mapsto [-1, 1]$ is defined as

\begin{equation}
    \text{sat}(x) = \begin{cases}
    1 & \text{if} \quad x \geq 1 \\
    x & \text{if} \quad -1 < x < 1 \\
    -1 & \text{if} \quad x \leq -1
    \end{cases}
\end{equation}

\noindent Therefore, $\eta$ is the internal state for the anti-windup compensator, which tracks the difference between attempted versus applied control input.
Tuning of the constants $K_p, K_I$, and $K_{AW}$ is discussed in Supplementary Information S1, Sec. 3.

\subsection{Sliding Mode Controller with Boundary Layer}

In addition to PI feedback as a standard approach, there has been much prior success in control of SMA-based robots and mechanisms using sliding mode control (SMC)~\cite{Elahinia2002,jin_continuous_2015,wiest_indirect_2014}.
SMC naturally addresses saturation issues, since switching occurs between some minimum and maximum input~\cite{slotine1991applied}.
We employ a model-free sliding mode controller with a boundary layer, as suggested by Elahinia et al.~\cite{Elahinia2002}, with a sliding surface $s$, using a finite-difference approximation of derivative as

\begin{align}
    \dot e(k) \approx & \frac{1}{\Delta_t} \left( e(k) - e(k-1) \right), \notag \\
    s(e) \vcentcolon= & \; \dot e(k) + 2\lambda e(k) +  K_I \left[ \sum_{\tau=0}^{k-1} e(\tau) \Delta_t \right] \\
    \mu(s) = & \; \text{sat}\left( \frac{s}{\phi} \right)
\end{align}

\noindent where $\phi \in \mathbb{R}^+$ is the boundary layer thickness and $\lambda \in \mathbb{R}^+$ is the phase plane angle~\cite{slotine1991applied}.
Tuning of this controller is also discussed in S1 Sec. 3.

\section{Supervisory Control Results}

We perform three sets of tests to characterize and validate the action of the supervisor on the above controllers.
In order to test our supervisor without permanently damaging the robot, we chose artificially low temperature constraints (between $60^\circ C$ and $90^\circ C$) for the tests.

\subsection{Theoretical Performance Verification}

We first confirmed that our framework functions as intended by implementing both the PI-AW and SMC controllers, and comparing their operation with versus without the supervisor.
We chose an arbitrary, but aggressive, step setpoint angle ($\bar \theta=40^\circ$ bending) and a maximum temperature of $w_i^{MAX}=65^\circ C$ for the supervisor, with $\gamma=0.2$ for conservative operation.

Fig.~\ref{fig:step} and Supplementary Video S2 show the results of all four tests: both controllers, with versus without the supervisor.
Both the PI-AW and SMC controllers, without the supervisor, regulate the limb around the desired setpoint with low error.
However, both controllers cause the SMA wire temperatures to drift upwards, representing potentially unsafe operation.
In contrast, the controllers with the supervisor included cause temperature to saturate at the maximum.
This is the intended behavior: the supervisor's activation sacrifices state tracking in favor of safe actuator states, and $w \leq w^{MAX}$ for all actuators.
We conclude from these tests that the supervisor indeed operates independently of the underlying pose controller $\mathbf{v}(\mathbf{x})$.

\begin{figure}[htbp]
    \centering
    \includegraphics[width=0.7\columnwidth]{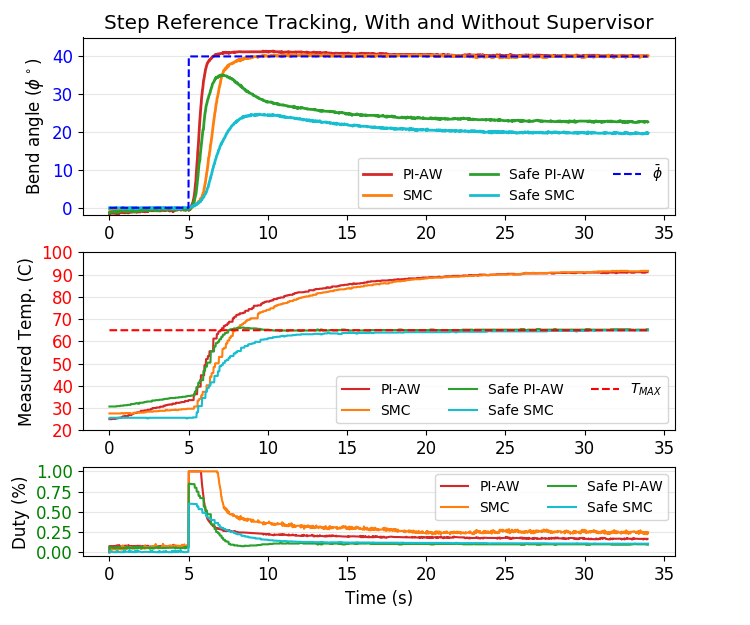}
    \caption{The two different pose controllers (PI-AW in red, SMC in orange) both stabilize around the desired angle, but may overheat the limb. However, imposing the supervisor ensures safe operating temperatures (blue, green) while attempting to reach the control goal.}
    \label{fig:step}
\end{figure}

\subsection{Supervisory Controller Tuning}

Our next test, with $\bar \theta=30^\circ$ and a max temperature of $w_i^{max}=60^\circ C$, executed the PI-AW controller with the supervisor for various values of its parameter $\gamma$.
We only test using the PI-AW pose controller, since Fig.~\ref{fig:step} shows similar activation behaviors for both pose controllers.

The data in Fig.~\ref{fig:supertuning} and Supplementary Video S3 demonstrate large variations in behavior depending on $\gamma$.
For small values ($\gamma \in [0.05, 0.2]$), the supervisor activates almost as soon as control begins ($t=5$ sec.), and temperature trajectories rise slowly but remain safe.
For larger values ($\gamma \in [0.3, 0.9]$), the supervisor is much less aggressive.
For the largest values, the limb briefly reaches the target $\bar \theta$ before the supervisor forces a lower input.
We do not report a $\gamma=1$ result since its performance was poor.

\begin{figure}[htbp]
    \centering
    \includegraphics[width=0.7\columnwidth]{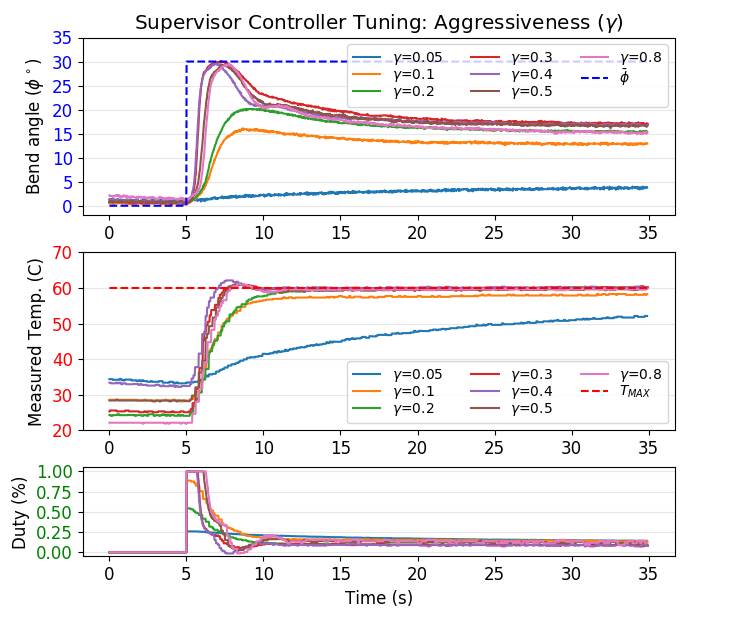}
    \caption{Controller tuning shows that large values of the parameter $\gamma \in (0,1)$, when the actuator dynamics are not known exactly, may slightly violate safety constraints. For practical uses we found that $\gamma < 0.3$ demonstrated safe operation.}
    \label{fig:supertuning}
\end{figure}

At larger values of $\gamma$, some violation of the safety constraint is observed, which is expected given our system identification of $\mathbf{A}$ in eqn.~(\ref{eqn:joule_affine}).
Unmodeled dynamics cause the value for $u^{MAX}$ to be artificially large.
This suggests that smaller $\gamma$ values are best in practice, since soft robot modeling is often imprecise.

\subsection{Physical Interactions}

We finally stress-test our feedback method in three different physical interaction scenarios, each representing an eventual use of our soft limb.
Each used the PI-AW controller for pose, with $\gamma=0.2$.
These three scenarios, in Fig.~\ref{fig:phys_int_screencaps}, include environmental contact, human contact, and the attempted tracking of infeasible/unsafe trajectories.
The first test (Fig.~\ref{fig:phys_int_screencaps}(a)) places a wall next to the limb that blocks it from reaching its target bend angle.
In the second test (Fig.~\ref{fig:phys_int_screencaps}(a)), a human pushes on the robot, causing a disturbance.

\begin{figure}[htbp]
    \vspace{0.5cm}
    \centering
    \includegraphics[width=6in]{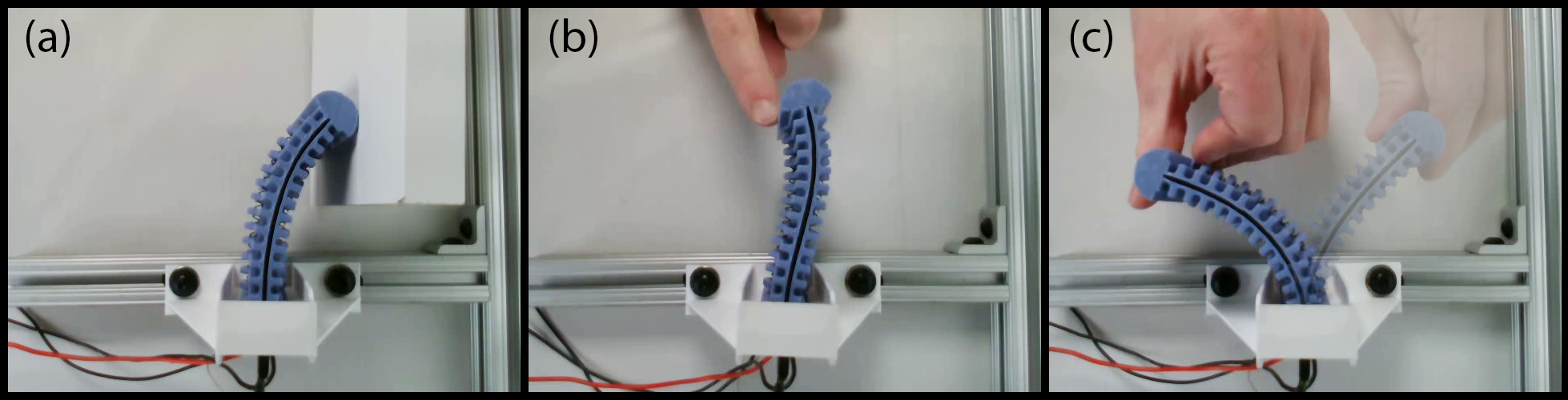}
    \caption{Three physical interactions that could cause damage to a soft robot under feedback: (a) contact and collision with kinematic constraints, such as a wall or floor, (b) unmodeled disturbances such as human interaction, and (c) attempting infeasible motions such as might unintentionally occur in learning from demonstration tasks.}
    \label{fig:phys_int_screencaps}
    \vspace{0.5cm}
\end{figure}

The final physical interaction example is tracking of a trajectory recorded beforehand by a human operator moving the limb, as in our prior work~\cite{wertz_trajectory_2022}.
Substituting the recorded trajectory as $\bar{\theta}_{1, \hdots, K}$ attempts to recreate that motion, however unsafe it may be.
With the supervisor, this procedure can be viewed as a crude form of \textit{learning from demonstration}~\cite{argall_survey_2009}, where a feedback controller mimics a demonstrated action under safety/feasibility constraints.

The data from these tests are in Fig.~\ref{fig:usecase_data}, and demonstrations are shown in Supplementary Videos S4-S6.
For the wall interaction, the unsafe control system without the supervisor heats the SMA wire rapidly and was manually deactivated before the test concluded, whereas the supervisor keeps the actuator at a steady maximum temperature.
For the human disturbance, the unsafe controller responded dynamically to the disturbance, causing continued heating, whereas the test with the supervisor prevented a changing input during those motions and implicitly bounded the force applied to the human.
Finally, for the ``learning from demonstration'' test, the unsafe controller was able to faithfully track the desired motion; however, the SMA temperatures violated constraints.
The corresponding test with the supervisor demonstrates it dynamically activating and deactivating as both wires reach potentially unsafe operation.

\begin{figure}[htbp]
    \vspace{0.5cm}
    \hspace{-1.75in}
    \includegraphics[width=7.5in]{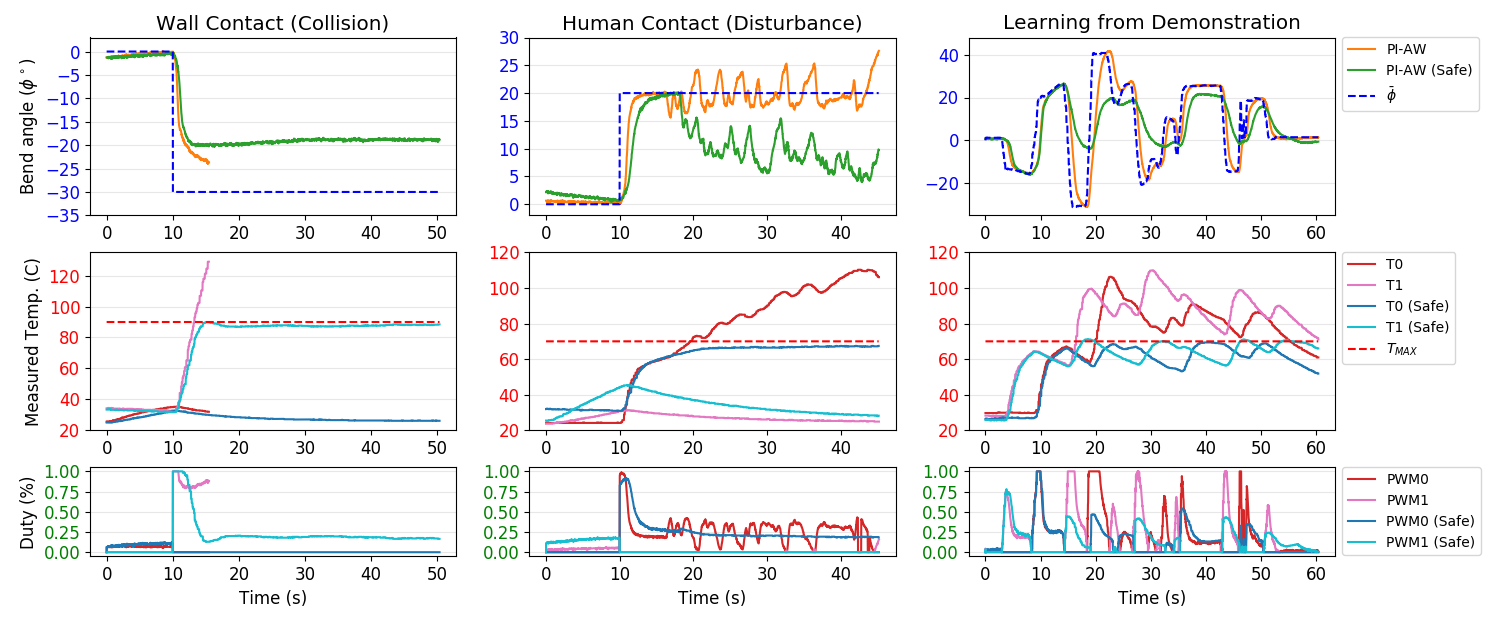}
    \caption{In each of the three physical interaction examples, the supervisory controller ensures safe operation when the robot's actuator state would otherwise violate safety constraints.}
    \label{fig:usecase_data}
    \vspace{0.5cm}
\end{figure}

\section{Discussion}

This article proposes a supervisory control scheme for a set of soft robot actuators that follow simple dynamics models, provably guaranteeing that actuator states remain in a safe region.
The proposed supervisor is simple to formulate and implement, with very low online computational cost.
Experiments show that the controller can be tuned for conservative operation even in the case when the actuator dynamics are a significant approximation, making the framework applicable for a variety of soft robot actuator designs.
We demonstrate that the controller safely operates on a thermal actuator, maintaining safe temperatures in a variety of contact-rich environments.

This work highlights the inherent relationship between force applied at a manipulator tip and the bounds on its actuator state, for example, in the human contact disturbance test.
Recent work has shown that environmental contact forces for a soft robot may be estimated simply from pose measurements~\cite{santina_datadriven_2020}.
Therefore, if a model for the body dynamics is available, it may be possible to convert between actuator bounds and body force bounds, allowing the concepts from this article to extend to safe interactions of body-to-environment: safety in pose as well as safety in actuator.

\section{Future Work and Conclusion}

Multiple directions of future work are anticipated to make the proposed framework more robust and applicable with fewer assumptions required.
In particular, a probabilistic actuator dynamics model (and accompanying modifications to the controller) may provide better robustness when the linearization is poor.
Similarly, future work will examine adaptive control for capturing unmodeled dynamics.
If the actuator dynamics cannot be linearized, we will examine optimization techniques for the supervisor, such as model-predictive control.
For robots with coupled actuator dynamics or more than one scalar state per actuator, future work may saturate actuators in terms of conic section bounds~\cite{angeli_monotone_2003}.

The system in this article relies on feedback for the supervisor, requiring sensors for all actuator states.
However, if the actuator dynamics model sufficiently captures the underlying physical phenomena, it may be possible to estimate the states $\widetilde{\mathbf{w}}(k)$ for use in the supervisor, eliminating the need for external sensing.

Lastly, a major motivation of this article is applying feedback control to soft robots in locomotion tasks.
We plan to implement our supervisor on SMA-actuated walking soft robots~\cite{patterson_untethered_2020} to demonstrate safe, closed-loop, locomotion in state-feedback.
Safe locomotion with feedback will bring soft robots closer to real-world deployment and positive demonstrations of their applicability to real-world tasks.


\section*{Acknowledgements} 

We thank Xiaonan Huang, Richard Desatnik, and all members of the Soft Machines Lab at CMU for their collaboration in the design framework for the robot studied in this article.

\section*{Author Contribution Statement}
APS: Conceptualization, Formal Analysis, Methodology, Software, Funding Acquisition, Writing - original draft. ZJP: Methodology, Software, Writing - review and editing. ATW: Methodology, Software, Writing - review and editing. CM: Funding Acquisition, Supervision, Writing - review and editing.

\section*{Author Disclosure Statement}
No competing financial interests exist.

\section*{Funding Information}

This work was in part supported by the Office of Naval Research under Grant No. N000141712063 (PM: Dr. Tom McKenna), the National Oceanographic Partnership Program (NOPP) under Grant No. N000141812843 (PM: Dr. Reginald Beach), and an Intelligence Community Postdoctoral Research Fellowship through the Oak Ridge Institute for Science and Education.

\section*{Supplementary Material}

Supplementary Information S1

Supplementary Video S2

Supplementary Video S3

Supplementary Video S4

Supplementary Video S5

Supplementary Video S6

\bibliography{references}

\end{document}